\documentclass{article}
\usepackage{ijcai26}

\usepackage{times}
\usepackage{soul}
\usepackage{url}
\usepackage[hidelinks]{hyperref}
\usepackage[utf8]{inputenc}
\usepackage[small]{caption}
\usepackage{graphicx}
\usepackage{amsmath}
\usepackage{amsthm}
\usepackage{booktabs}
\usepackage[ruled,linesnumbered,vlined]{algorithm2e}

\usepackage[switch]{lineno}
\usepackage{amssymb}
\usepackage{multirow}
\usepackage{microtype}
\usepackage{tabularx}
\newcommand{\descr}[1]{#1}
\newtheorem{definition}{Definition}
\newtheorem{lemma}{Lemma}
\usepackage{xcolor}
\usepackage{enumitem}
\newtheorem{theorem}{Theorem}
\title{Parameterized and Streaming Algorithms for Euclidean Fair $k$-Center Clustering}

\author{
Zeyu Lin$^1$
\and
Chaoqi Jia$^2$\and
Longkun Guo$^{1}$\thanks{Corresponding Author}\And
Chao Chen$^2$\\
\affiliations
$^1$School of Mathematics and Statistics, Fuzhou University, Fuzhou 350116, China\\
$^2$School of Accounting, Information Systems and Supply Chain, RMIT University, VIC 3000, Australia
\emails
zeyu\_lin@foxmail.com,
chaoqi.jia@rmit.edu.au,
longkun.guo@gmail.com, 
   chao.chen@rmit.edu.au
}

\begin{document}

\maketitle

\begin{abstract}

Motivated by the growing importance of fairness in machine learning, fair $k$-center clustering has attracted considerable research attention as a fundamental problem. In this problem, a dataset is partitioned into $m$ disjoint groups, and the objective is to select $k$ data points as centers, subject to upper bounds on the number of centers chosen from each group, aiming to minimize the maximum distance between any data point and its assigned center. Focusing on Euclidean spaces, which are ubiquitous in machine learning applications, we first develop a parameterized approximation algorithm for Euclidean fair $k$-center with an approximation ratio of $2.732$. By incorporating this algorithm as a post-processing stage into a one-pass streaming framework for large-scale data, we obtain an approximation ratio of $4.464$. These ratios can be further respectively improved to $2.414$ and $3.828$ with a runtime exponential on $k$. To ensure polynomial-time complexity, we further design a one-pass streaming algorithm with an approximation ratio of $4.732$, which can be further improved to $4.42$, outperforming the state-of-the-art ratio. Finally, extensive experiments show that our methods significantly outperform state-of-the-art approaches in terms of clustering accuracy.
\end{abstract}

\section{Introduction}

Center-based clustering is a typical unsupervised learning technique in machine learning (ML) with a wide range of applications, including data summarization~\cite{gadekar2025fair}, text analysis~\cite{jia2026textcluster}, and decision-making~\cite{huang2019coresets}. In particular, the $k$-center problem aims to minimize the maximum distance from any point to its nearest center and has been extensively studied. In general metric spaces, a tight $2$-approximation is achievable~\cite{gonzalez1985clustering,hochbaum1985best}, whereas in Euclidean spaces, it is NP-hard to obtain a polynomial-time approximation with a factor better than $1.82$~\cite{feder1988optimal}.

Many applications have proposed $k$-center problem with constraints including instance-level constraints ~\cite{Guo2024Efficient,guo2025near} for improve the clustering accuracy and  \emph{fairness} constraints  to mitigate representation bias~\cite{chierichetti2017fair,bera2022fair}. In particular, we focus on data summarization fairness which can limit the number of older messages included in a user's feed digest~\cite{mahabadi2023core} or enforce genre diversity in movie recommendations. For data summarization fairness,  the dataset $S$ is partitioned into $m$ disjoint groups $S = S_1 \cup \cdots \cup S_m$ and it requires that at most $k_l$ centers be selected from each group $S_l$. This constraint in the fair $k$-center ensures balanced representation by sensitive attributes. 

Under this fairness setting, there has been significant research interest in improving approximation ratios for the problem across different metrics, such as general metrics~\cite{guo2026improved}, Euclidean metrics~\cite{guo2026fair} and doubling metrics~\cite{ceccarello2024fast}, as well as under different computational models, including the offline setting~\cite{kleindessner2019fair,jones2020fair} and the streaming model~\cite{chiplunkar2020solve,lin2024streaming,guo2026improved}.

Inspired by \citeauthor{nagarajan2020euclidean}, we recently achieved a $(1+\sqrt{3})$-approximation for the Euclidean $k$-supplier problem, improving over the long-standing factor-$3$ barrier. 
In Euclidean space, we combine with both \citeauthor{jones2020fair} and \citeauthor{guo2026improved} to design algorithms for the fair $k$-center problem in the streaming model and achieve a $(3+\sqrt{3}+\epsilon)$-approximation guarantee. After that, we further aim to improve the approximation ratio. To this end, we propose the first fixed-parameter tractable (FPT) algorithm for fair $k$-center in Euclidean space, breaking the existing approximation barrier while handling both online model and offline setting.

\subsection{Related Work}

\paragraph{FPT $k$-Center under Constraints and Specialized Metrics.}
Fixed-parameter tractable (FPT) approximation algorithms have been extensively studied for the $k$-center problem in metrics. {\citeauthor{agarwal2002exact}} gave an approximation scheme with running time $O(n\log k)+(k/\epsilon)^{O(k^{1-1/d})}$
where $n$ is the input size and $d$ is the dimension. More recently, {\citeauthor{goyal2023tight}} established a unified FPT framework for a broad class of constrained $k$-center problem, obtaining $2$-approximation algorithms with runtime $k^{O(k)}n^{O(1)}$.
For  constrained variants, \citeauthor{bandyapadhyay2024parameterized} obtained a $3$-approximation for  non-uniform $k$-center with a runtime $2^{O(k\log k)}n^2$.
{\citeauthor{wu2023fair}} further studied fair $k$-center with outliers and obtained an approximation scheme with runtime $f(k,z,\epsilon)\cdot n^{O(1)}$, where
$f$ is a computable function depending on $k$, $z$ and $\epsilon$, $z$ denotes the maximum number of outliers.
For Euclidean spaces, {\citeauthor{bandyapadhyay2024parameterized}} also developed an approximation scheme for Euclidean non-uniform $k$-center with runtime $2^{O((k\log k)/\epsilon)}dn$.
For FPT algorithms  beyond metric spaces, {\citeauthor{zhou2025metric}} studied $k$-center on general graphs and proposed a $3$-approximation algorithm running in $O(2^{\alpha}\cdot n^5)$ time
as well as a $13$-approximation algorithm running in $\beta^{O(\beta)}\cdot poly(n)$
where $\alpha$ is the number of vertices participating in triangle-inequality violations and $\beta$ is the minimum number of vertices whose removal transforms the graph into a metric graph.

\paragraph{Data Summarization Fair $k$-Center.}
Fair clustering under group representation constraints, specifically, upper bounds on the number of centers per group, was formalized in the context of data summarization by \citeauthor{kleindessner2019fair}, who gave a $5$-approximation for two groups and a $(3 \cdot 2^{m-1} - 1)$-approximation for $m$ groups. This was significantly improved by \citeauthor{jones2020fair}, who achieved a $3$-approximation for arbitrary $m$ via a reduction to matching. \citeauthor{chen2024approximation} later studied a closely related diversity-aware fair $k$-supplier problem and also obtained a $5$-approximation using maximum matching. These offline results establish $3$ as the state-of-the-art approximation ratio in general metrics, and a natural question is whether this barrier can be broken in structured spaces $\mathbb{R}^d$.

However, all known algorithms achieving ratios below $5$ are either offline or require multiple passes; none operate in the one-pass streaming regime while exploiting Euclidean geometry or fixed-parameter tractability. 

\paragraph{Streaming Algorithms for Fair Clustering.}
Streaming fair clustering poses additional challenges due to the irrevocable nature of decisions and limited memory. Early work on streaming $k$-center with outliers~\cite{charikar2003better,matthew2008streaming} laid foundational techniques, but fairness constraints require more sophisticated summaries. \citeauthor{lin2024streaming} recently proposed a one-pass streaming algorithm for fair $k$-center achieving a $(7+\epsilon)$-approximation, while \citeauthor{chiplunkar2020solve} gave a two-pass streaming algorithm with ratio $3$. By leveraging a $\lambda$-independent center set and reducing the hardest case to a constrained vertex cover problem, \citeauthor{guo2026improved} presented a one-pass streaming algorithm that achieved a $(5+\epsilon)$-approximation for fair $k$-center in general metric spaces, improving over the previous $(7+\epsilon)$ bound. {For Euclidean space, \citeauthor{guo2026fair} further improved the approximation ratio to $(1+2\sqrt{3}+\epsilon)\approx 4.464$.} Notably, none of the prior streaming approaches leverage FPT subroutines or geometric coreset constructions to achieve better ratios, leaving a gap between offline Euclidean advances and streaming practice.

To the best of our knowledge, our work is the first to combine FPT techniques, Euclidean geometry, and streaming models to achieve better approximation ratio for fair $k$-center in a one-pass streaming setting.

\subsection{Our Contribution}
 
Motivated by the importance of Euclidean space in machine learning, we propose a suite of parameterized, polynomial-time approximation and streaming algorithms. The contribution can be summarized as follows:

\begin{itemize}[leftmargin=1em]
\item Devise a parameterized approximation algorithm with ratio $1+\sqrt{3}\approx2.732$ and integrate it into the streaming framework to obtain an approximation ratio  $1+2\sqrt{3}+\epsilon \approx 4.464$ for Euclidean fair $k$-center. These guarantees can be further improved to $1+\sqrt{2}\approx2.414$ and $1+2\sqrt{2}+\epsilon\approx 3.828$, respectively.

\item Propose a polynomial-time one-pass streaming algorithm for Euclidean fair $k$-center that uses $O(k\log \alpha)$ memory and achieves an approximation ratio of $3+\sqrt{3}+\epsilon\approx 4.732$ by integrating the milestone algorithm of \cite{jones2020fair}, where $\alpha$ is the aspect ratio, defined as the ratio between the maximum and minimum pairwise distances.

\item Conduct extensive experiments on real-world datasets to evaluate the practical performance of our algorithms, demonstrating that they significantly outperform all state-of-the-art methods  in terms of  solution quality.

\end{itemize}

{

Moreover, with a higher memory complexity of $O(h\cdot k\log \alpha)$, the approximation ratio of our streaming algorithm can be further improved to $3+\sqrt{2+\frac{2}{h}}\approx 4.42$ by setting $h=122$, which slightly improves upon the state-of-the-art ratio of $4.464$ due to~\citeauthor{guo2026fair}.}

\section{Preliminary}
\label{sec:pre}
Let $S$ be a finite set of $n$ data points distributed in an Euclidean space, where the distance function $d: S \times S \rightarrow \mathbb{R}_{\geq 0}$. For a given parameter $k \in \mathbb{N}$, the traditional $k$-center clustering problem seeks to select a set of $k$ centers $C \subseteq S$ such that $\max_{s \in S} d(s, C)$ is minimized, where $d(s, C) = \min_{c \in C} d(s, c)$ denotes the distance from a point $s$ to its nearest center in $C$.

\subsection{Problem Statement}
The fair $k$-center clustering problem extends this by imposing additional fairness constraints. Specifically, the dataset $S$ is divided into $m$ disjoint groups: $S = S_1 \cup S_2 \cup \cdots \cup S_m$. Each group $S_l$ ($l \in [m]$, where $[m] = \{1, 2, \ldots, m\}$) has an associated upper bound $k_l$ on the number of centers that can be selected from it. These constraints ensure that the total number of chosen centers across all groups equals $k$: $\sum_{l=1}^m k_l = k$. This formulation aims to balance the representation of each group in the final clustering solution, thereby mitigating potential biases and ensuring fairness. Then,  the fair $k$-center clustering problem is to find a center set $C$ satisfying the formulation as follows: 
\begin{align}
\min_{C \subseteq S}  && \max_{s\in S}d\left(s,C\right) &&\nonumber \\
s.t. && \mid C\cap S_{l}\mid&\leq k_{l},\, \forall l &\label{eq:fair}\\
&&\mid C\cap S\mid&=\sum_{l=1}^{m} \mid C\cap S_{l}\mid \leq k. &
\end{align}

We assume that the optimal radius of the fair $k$-center problem is known  as $r^*$ (i.e., $r^{*}= \max_{s\in S}d\left(s,C^*\right)$ for the optimum center set $C^*$). Note that we actually do not know the exact value of $r^*$. 
 Nevertheless, we devise $r^*$ following the previous works~\cite{Jia2026Approximation} for the offline algorithm and \cite{guha2009tight} for the online algorithm.

\subsection{\texorpdfstring{$\lambda$-Independent Center Set in Euclidean Space}{lambda-Independent Center Set in Euclidean Space}}

A key structural tool in this paper is the $\lambda$-independent center set, a concept formalized in the streaming setting by \cite{guo2026improved}. Its definition is as follows:

\begin{definition}
\label{def:1}($\lambda$-independent center set) $\Gamma\subseteq S$ is a $\lambda$-independent center set of $S$, if and only if it satisfies the following two conditions:
\begin{itemize}[leftmargin=1.3em]
    \item[ 1)] For  any two points  $p,q\in \Gamma$, the distance between them is larger than $\lambda$, i.e. $d(p,q)>\lambda$.
    \item[ 2)] For any point $p\in S$, there exists a point  $q\in \Gamma$, such that $d(p,q)\leq \lambda$.
\end{itemize}
\end{definition}

Recall that $C^*$ is the optimal center set and $r^{*}$ is the optimum radius. Then we have:

\begin{lemma}
    \label{lem:sqrt3d-independent center set} 
    Assume  $\Gamma\subseteq S$ is a $\lambda$-independent center set in Euclidean space.
 If $\lambda = \sqrt {3}r^{*}$, then $|\Gamma|\leq 2k$. 
\end{lemma}
\begin{proof}
    We first show that the number of points in $\Gamma$ within distance $r^{*}$ from each $c^*\in C^*$ is at most two. Suppose that there exists a center $c^*$ with three points $c_{1},c_{2},c_{3}\in \Gamma$. Then we have $d(c^{*},c_{i})\leq r^{*}$ for
$i\in\left\{ 1,2,3\right\} $. That is, the circle centered at
$c^*$ with radius $r^{*}$  has $\left\{ c_{1},c_{2},c_{3}\right\} $
in its interior. Then, by the geometric property of any inner triangle within a circle, we have  
\[\max \underset{\forall i,j\in\left\{ 1,2,3\right\}}{\min}d(c_{i},c_{j}) =\sqrt{3}r^*,\] 
where the maximum is attained when all the three points $ c_{1},c_{2},c_{3}$ are exactly on the circle and the distance between any two points equals $\sqrt{3}r^*$. This contradicts the fact that the
distance between any pair of points in $\Gamma$ is strictly greater than $\sqrt{3}r^{*}$. So there are at most two points of $\Gamma$ that are within distance $r^{*}$ from any $c^*\in C^*$. Therefore, the size of $\Gamma$ is at most $2|C^*|=2k$.
\end{proof}
Moreover, when $\lambda = 2r^{*}$, we can  prove $|\Gamma|\leq k$ similarly.

\section{\texorpdfstring{Parameterized and Streaming Algorithms for Euclidean Fair $k$-Center}{Parameterized Algorithms for Fair k-Center in Euclidean Space}}
\label{sec:Offline}

 In this section, we first show that the $\lambda$-independent center set can lead to a parameterized approximation algorithm with approximation ratio $1+\sqrt{3}\approx 2.732$ in Euclidean space, which is better than $3$ for the metric space when requiring a runtime exponential on $k$.  
Then, we propose incorporating the parameterized algorithm into the post-processing stage of the streaming framework for fair $k$-center clustering, achieving an overall approximation ratio of $1+2\sqrt{3}+\epsilon$, which is approximately  $4.464$ for small $\epsilon$. Moreover, we show that the ratio of the offline parameterized approximation algorithm can be further improved to $2.414$, and consequently, the ratio of the streaming algorithm is then improved to $3.828$. The proof of the improvement will be provided in the full version.

\subsection{Offline Parameterized Approximation}
\label{subsec:off_approx}

First, we show that $\lambda$-independent center sets with $\lambda = \sqrt{3}r^*$ can be employed to derive a parameterized approximation algorithm as $\bigcup_{l=1}^{m} \Gamma_l$ contains a desirable approximation solution, where $\Gamma_l$ is a $\lambda$-independent center set for $S_l$ with $\lambda=\sqrt{3}r^*$:

\begin{theorem}
\label{thm:the property of sqrt3r-independent set}
    For $\lambda=\sqrt{3}r^*$, there exists $\Gamma \subseteq \bigcup_{l=1}^{m} \Gamma_l$ for which the following three conditions hold:  (1) for each point $s\in S$, there must exist a point $c\in \Gamma$ with $d(s,c)\leq (1+\sqrt{3})r^*$; (2) $\vert \Gamma\vert \leq k$; (3) $|\Gamma\cap S_l|\leq k_l$.   In other words,   there exists $\Gamma$ in $\bigcup_{l=1}^{m} \Gamma_l$  that is a  $(1+\sqrt{3})$-approximation solution for the fair $k$-center problem.             
\end{theorem}

\begin{proof}
Let $C^{*}$ be an optimal solution. We first give a method to construct a $\Gamma$ from $\bigcup_{l=1}^{m} \Gamma_l$ according to $C^{*}$, and then show that the constructed  $\Gamma$ satisfies all three conditions. 

For the first, our construction simply proceeds as in the following: for each center $c^{*}\in C^{*}$, if $c^{*}\in C^{*}$ is an element of $\bigcup_{l=1}^{m} \Gamma_l$, then add  $c^{*}$ to $\Gamma$; Otherwise, 
find group $S_l$ that contains $c^*$, and then find a point $c\in \Gamma_l$ which is within $\sqrt{3}r^{*}$ distance away from $c^{*}$, i.e. with $d(c,c^*)\leq \sqrt{3}r^{*}$. So, for each $s\in S$, there must exist $c\in\Gamma$ such that $d(s,c)\leq d(s,c^*)+d(c^*,c)\leq r^*+\sqrt{3}r^*=(1+\sqrt{3})r^*$, and Cond.~(1) holds. As the construction maps each optimal center $c^*\in C^*\cap S_l$ to one center $c\in \Gamma\cap S_l$, Cond.~(2) and Cond.~(3) hold according to the property of the optimal center set $C^*$. 
\end{proof}

 Based on this existence guarantee, our enumeration-based parameterized algorithm simply proceeds in two stages: 
\begin{itemize}[leftmargin=1.4em]
    \item [1)] For each group $l \in [m]$, compute a $\lambda$-independent center set $\Gamma_l$ with $\lambda = \sqrt{3}r^*$.

\item [2)] Enumerate all the possible center sets  $C$ with size bounded by $k$ and satisfying the fairness constraint, and choose $C$ with the smallest radius among all the computed center sets.
\end{itemize}

\begin{lemma}
\label{lem:algosqrt3r-independentset}  
   The algorithm above can correctly find $\Gamma$ that satisfies the three conditions of Thm.~\ref{thm:the property of sqrt3r-independent set}.                   
\end{lemma}

\begin{proof}
    First, we show the algorithm can always output a $\Gamma$, which is, there always exists $c\in \Gamma_l$ with $d(c,c^*)\leq \sqrt{3}r^{*}$ when $c^{*}\notin \bigcup_{l=1}^{m} \Gamma_l$ and $c^{*} \in S_l$. Suppose otherwise, then $d(c^{*},c)> \sqrt{3}r^{*}$ holds for every $c\in \Gamma_l$, which indicates that $c^{*}$ must be added to $ \Gamma_l$ and arises a contradiction.

Then, we show that all three conditions hold. Cond.~(1) immediately holds following the construction of $\Gamma$.

For Cond.~(2), we add at most one point (i.e., $c$ or $c^*$) to $\Gamma$ upon one point $c^*\in C^*$ (note that the point to be added might already exist in  $\Gamma$). So $|\Gamma|\leq |C^*|= k$ and hence condition (2) is true. 
 For condition (3), the algorithm adds a point (either $c$ or $c^*$) of $S_l$ to $\Gamma$ if and only if the processing point $c^*$ belongs to $S_l$. That is, we have $|\Gamma\cap S_l| \leq  |C^*\cap S_l|\leq k_l$ holds for each $l$. This completes the proof.
\end{proof}

\begin{theorem}\label{thm:paramenterbest}
    The fair $k$-center problem admits a parameterized approximation algorithm with ratio $1+\sqrt{3}\approx 2.732$ and runtime $O(2^{k(1+\log m)}\cdot mnk^2\log n)$.
\end{theorem}

\begin{proof}
    The ratio can be directly derived from Thm. \ref{thm:the property of sqrt3r-independent set}. That is because the enumeration can always find a $\Gamma$ satisfying the three conditions as in Thm. \ref{thm:the property of sqrt3r-independent set} if such $\Gamma$ exists; while on the other hand, Thm. \ref{thm:the property of sqrt3r-independent set} claims the existence of such  $\Gamma$.  

For the runtime, Stage 2 enumerates $O\left(\tbinom{2mk}{k}\right)$ sets (of $k$-centers), because there are at most $2mk$ points in $\bigcup_{l\in [m]}\Gamma_l$. Moreover, each enumerated set requires $O(nk)$ time to compute its maximum radius, so Stage 2 takes a total runtime $O\left(nk\cdot\tbinom{2mk}{k}\right)$. By  employing Stirling's approximation  formulation and via calculation, we get: 
\begin{lemma}\label{lem:uppertbinom}
    $\tbinom{2mk}{k}$ is bounded by 
    $O\left(2^{k(1+\log m) }\cdot mk\right)$.
\end{lemma}
Then, the total runtime of the parameterized algorithm sums up to $O(2^{k(1+\log m)}\cdot mnk^2\log n)$. Then by multiplying the factor $O(\log n)$ for employing the binary search for not knowing $r^*$, we complete the proof of Thm.~\ref{thm:paramenterbest}. 
{The detailed proof of Lem.~\ref{lem:uppertbinom} can be found in the full version.}

\subsection{The Parameterized Approximation in Streams}
\label{subsec:approx-1_2sqrt3}
In this subsection, we introduce the general framework of our parameterized streaming algorithm, there are two stages:

\begin{itemize}[leftmargin=1.4em]
    \item [1)] \textbf{Streaming Stage}. Construct a $\lambda$-independent center set $\Gamma_{l}$ for each group $S_l$ {\it{upon the stream}} with $\lambda=\sqrt{3}r^*$.

\item [2)] \textbf{Post-streaming Stage}. Select centers from $\bigcup_{l=1}^{m} \Gamma_l$ to construct the desired center set $\Gamma$ by enumeration similar to the offline parameterized approximation algorithm.
\end{itemize}

Because the enumeration here considers only the points in $\bigcup_{l}\Gamma_l$ rather than in the entire point set $S$,
its compromised performance guarantee can be stated as below:

\begin{lemma}\label{lem:1plus2sqrt3}
    The parameterized streaming algorithm consumes a memory  of $O(k\log_{\epsilon} \alpha)$ and achieves an approximation ratio $1+2\sqrt{3}+\epsilon\approx 4.464$ for any sufficiently small $\epsilon>0$, where $\alpha = \Delta / \delta$ is the \emph{aspect ratio} for $\Delta = \max_{p,q \in S} d(p,q)$ denoting the diameter of $S$ and  $\delta = \min_{p,q \in S,\,p \neq q} d(p,q)$ being the minimum pairwise distance. 
\end{lemma}

\begin{proof}
        We need only to bound the distance from any point $s\in S$ to the center set $C$ constructed by the parameterized algorithm. For each input point $s\in S_l$, there must exist a point $i\in \Gamma_l$ such that $d(i,s)\leq \sqrt{3}r^*$ according to the structure for $\lambda$-independent center set. If $i$ is  added to $C$, then $d(s,C)\leq \sqrt{3}r^*$ immediately holds. Otherwise, we have $i\notin C$ and will show $d(s,C)\leq (1+2\sqrt{3})r^*$ holds. 
    Following Thm.~\ref{thm:the property of sqrt3r-independent set}, there exists $\Gamma\subseteq \bigcup_l \Gamma_l$, such that for any $i\in \bigcup_l \Gamma_l$, $d(i,\Gamma)\leq (1+\sqrt{3})r^*$ holds.  As   $C$ is constructed via enumeration, we have 
    \[\max_{i\in  \bigcup_l \Gamma_l} d(i,C) 
    \leq \max_{i\in  \bigcup_l \Gamma_l } d(i,\Gamma)\leq (1+\sqrt{3})r^*.\]
    Therefore, we have 
    \[d(s,C)\leq d(s,i)+\underset{i\in  \bigcup_l \Gamma_l }{\max}d(i,C)\leq (1+2\sqrt{3})r^*,\] 
    
    This completes the proof. 
\end{proof}

Note that, by using the smaller value $\lambda=\sqrt{2}r^*$, the parameterized approximation ratio in Thm.~\ref{thm:paramenterbest} and the ratio in Lem.~\ref{lem:1plus2sqrt3} can be improved to $1+\sqrt{2}\approx 2.414$ and $1+2\sqrt{2}+\epsilon\approx 3.828$, respectively.

\section{Streaming Euclidean Fair $k$-Center in  Polynomial Runtime}
\label{sec:edu_fair}
In this section, we present a one-pass streaming algorithm for  Euclidean space, where the algorithm employs network flow as a building block, achieving a polynomial runtime and a ratio of $3+\sqrt3+\epsilon\approx 4.732$.

\begin{algorithm}[!t]
 \small

\caption{The post-streaming algorithm }
\label{alg:jones_based_stream}

\KwIn{A $\lambda$-independent center set $\Gamma$  with $\lambda = 2r^*$ and  $\Gamma_{i}$ ($i\in[m]$) with $\lambda = \sqrt{3}r^*$.}
\KwOut{A center set $C$.}

Create a directed graph $G(V,E)$ with $V = \{s, t\}$ and $E = \emptyset$\;

Create vertex set $V_{\Gamma}$ with one-to-one mapping to points in $\Gamma$\;
Create a vertex set $V_f$ with a one-to-one mapping to points in $\Gamma_1, \dots, \Gamma_m$\;
Create a vertex set $V_t$ of size $m$ with a one-to-one mapping to group indices\;
Update $V \gets V \cup V_{\Gamma} \cup V_f \cup V_t$\;

\For{each group $l \in [m]$}{
    Add an edge from the vertex corresponding to $l$ in $V_t$ to $t$ with capacity $k_l$\;
}

\For{each $a \in \Gamma$ and each $f \in V_f$ corresponding to a point in $\bigcup_{l \in [m]} \Gamma_l$}{
    \If{$d(a, f) \leq (1 + \sqrt{3}) r^*$}{
        Add an edge from vertex $a \in V_{\Gamma}$ to vertex $f \in V_f$ with capacity $1$\;
        Add an edge from vertex $f \in V_f$ to the vertex $l \in V_t$ (where $f \in \Gamma_l$) with capacity $1$\;
    }
}

\For{each $a \in \Gamma$}{
    Add an edge from $s$ to the vertex for $a \in V_{\Gamma}$ with capacity $1$\;
}

Run Dinic's Max-Flow algorithm~\cite{Shimon1975Network} on $G$ from $s$ to $t$\;

\eIf{the max flow value equals $|\Gamma|$}{
    Construct $C$ from the centers corresponding to edges used in the flow\;
    \Return $C$\;
}{
    \Return \textit{infeasible}.
}
\end{algorithm}

The key idea of our algorithm is first to construct a $\lambda$-independent center set $\Gamma$ for $S$ with $\lambda=2r^*$(ignoring groups) and for each group $l$, construct $\lambda$-independent center set $\Gamma_{l}$ with a different $\lambda=\sqrt{3}r^*$ along the stream simultaneously; and then use the $\Gamma$ as the solution except replacing some centers therein using $\bigcup_l{\Gamma_l}$ to satisfy the fairness constraint.  

In general, our algorithm mainly proceeds in two phases:

 (1) Upon the stream, construct $m+1$ independent center sets, including $m$ independent center sets $\Gamma_{l}$ for each group $S_{l}$ ($l\in[m]$) regarding $\lambda= \sqrt{3}r^*$ and an independent center set $\Gamma$ for $\lambda= 2r^*$ ignoring the group division; 
 
 (2) Select at most $k$ centers from the computed $\bigcup_l{\Gamma_l} \cup \Gamma$: Construct an auxiliary bipartite graph $G$ where a maximal flow corresponds to a center set that can cover all the data points within the desired radius $(3+\sqrt{3})r^*$.

The auxiliary graph $G=(V;E)$ for Phase (2) can be constructed as follows: 
 \begin{itemize}[leftmargin=1.4em]
\item [1)] For the vertices of $G$, set $V=\{s,t\}\cup V_{\Gamma}\cup V_f\cup V_t$ where $\{s,t\}$ is the set of source and target nodes, $V_{\Gamma}$ is a set with a one-to-one mapping to points in $\Gamma$, $V_f$ is a set with a one-to-one mapping to points in $\bigcup\Gamma_l$ and $V_t$ is a set with a one-to-one mapping to group indices. 
\item [2)] We add four kinds of edges of $G$ as follows:
\begin{itemize}[leftmargin=1em]

\item For each $a \in V_{\Gamma}$, we add an edge from the source node $s$ to the vertex $a$ with capacity $1$.

\item For each $a \in V_{\Gamma}$ and each $f \in V_f$,  
    if $d(a, f) \leq (1 + \sqrt{3}) r^*$, we add an edge from vertex $a$ to vertex $f$ with capacity $1$.

 \item For each $f \in V_f$, we add an edge from vertex $f$ to the vertex $l \in V_t$ (where $f \in \Gamma_l$) with capacity $1$.
  
  \item For each group $l \in [m]$, we add an edge from the vertex corresponding to $l$ in $V_t$ to the target node $t$ with capacity $k_l$.
\end{itemize}
\end{itemize}

The detailed algorithm is as depicted in Alg.~\ref{alg:jones_based_stream} and an example is in Fig.~\ref{fig:the matching}.

\begin{figure}[t]
    \centering
    \includegraphics[width=0.4\textwidth]{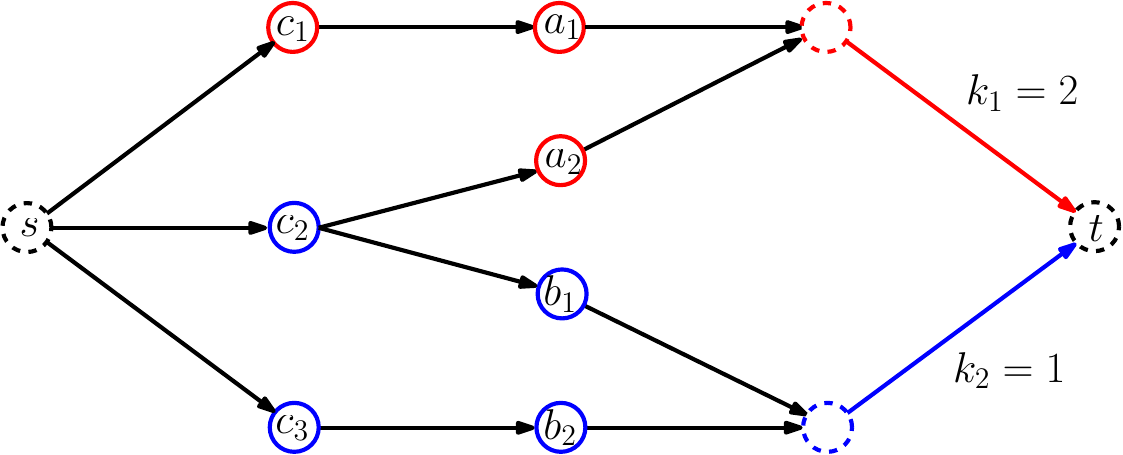}
    \caption{An example of executing Alg.~\ref{alg:jones_based_stream} with $k=3$, $k_1=2,k_2=1$. Let  $\Gamma= \{c_1, c_2, c_3\}$ be a $\lambda$-independent center set for $S$ ignoring the groups for $\lambda=2r^*$. Assume that $\Gamma_1=\{a_1, a_2\}, \Gamma_2=\{b_1, b_2\}$ are for $S_1, S_2$ respectively with a different $\lambda=\sqrt{3}r^*$. Each edge has a default capacity of $1$ except the edges entering $t$.}
    \label{fig:the matching}
\end{figure}

\begin{lemma}
\label{lem:the size of gamma is less than k}
    In Alg.~\ref{alg:jones_based_stream}, two conditions hold: (1) The size of each $\Gamma_{l},l\in [m]$, is less than $2k$. (2) The size of $\Gamma$ is less than $k$. 
\end{lemma}

\begin{proof}
    According to Lem.~\ref{lem:sqrt3d-independent center set}, 
    $|\Gamma|\leq k$ is true and $|\Gamma_l|\leq 2k$ holds for each $l\in[m]$. 
\end{proof}

\begin{figure*}[t]
  \centering
  \includegraphics[width=0.96\textwidth]{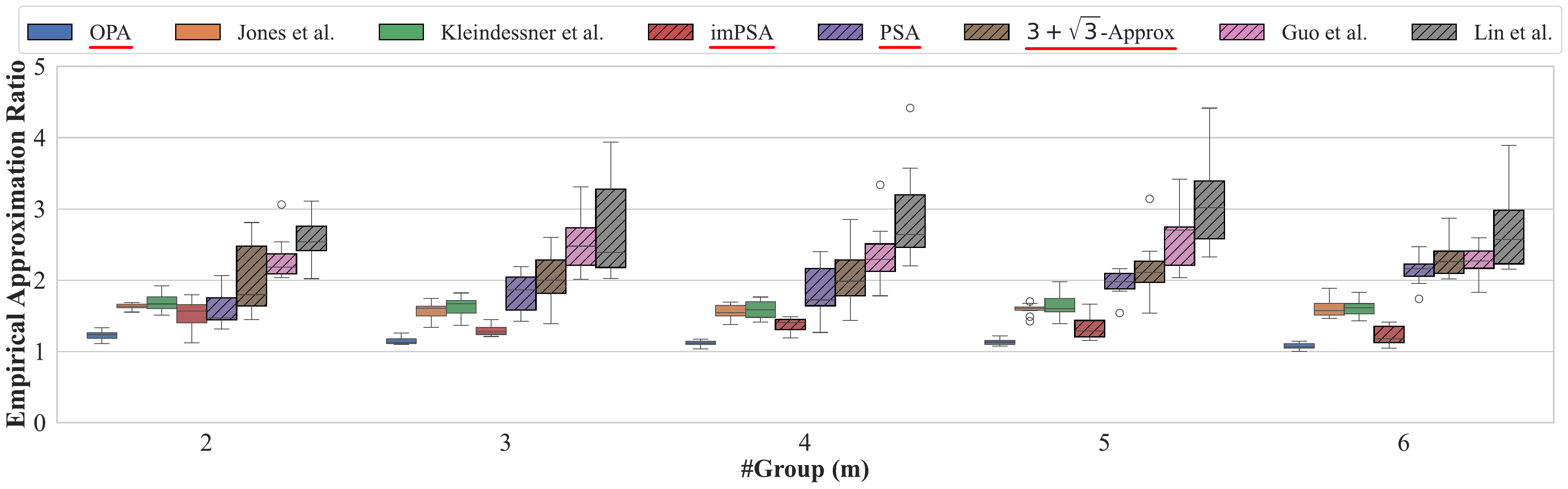}  
  \caption{Empirical approximation ratio ($\mathrm{cost}/r^*$) of our algorithms (indicated by a red underline) compared with other baselines. (Solid colors indicate the offline algorithms' result, whereas hatched boxes represent the results of streaming algorithms.)}
  \label{fig:simu_approx}
\end{figure*}

\begin{lemma}
\label{lem:the property of 1+sqrt3}
    In Alg.~\ref{alg:jones_based_stream}, for point $p\in \Gamma$, if $\forall i \in \Gamma_{l}, d(p,i)>(1+\sqrt{3})r^*$ holds where $l\in [m]$, there exists no point $q\in S_{l}$ such that $d(p,q)\leq r^*$.  
\end{lemma}

\begin{proof}
    Suppose that there exists a point $q\in S_{l}$ such that $d(p,q)\leq r^*$. Then, there must exist a point $i\in \Gamma_l$ such that $d(q,i)\leq \sqrt{3}r^*$. We have $d(p,i)\leq d(p,q)+d(q,i)\leq r^*+\sqrt{3}r^*=(1+\sqrt{3})r^*$, contradicting $d(p,i)>(1+\sqrt{3})r^*$.
\end{proof}

\begin{lemma}
\label{lem:feasible solution}
    There exists a flow of value equal to $\vert\Gamma\vert $ in the auxiliary graph $G$ if and only if there exists a center set $C\subseteq \bigcup_l{\Gamma_l} \cup \Gamma$ satisfying the fairness requirement and covering all the points of $S$ within a radius $(3+\sqrt{3})r^*$.
\end{lemma}

\begin{proof}
       
        If there exists a flow of value equal to $|\Gamma|$ in the auxiliary graph $G$, we need to prove $C$ satisfies: (1) $|C|\leq k$; (2) $|C\cap S_l|\leq k_l$; (3) $\forall s\in S, d(s,C)\leq (3+\sqrt{3})r^*$. 
        
        For $|C|\leq k$, the weight of edge from the vertex $a$ in $V_{\Gamma}$ to the vertex $f$ in $V_f$ is $1$, that is the exchange is one-to-one, when the $C$ created by the maximal flow algorithm, $|C|=|\Gamma|\leq k$ as $\Gamma$ is an $2r^*$-independent center set for $S$. 
        
        Then, for $|C\cap S_l|\leq k_l$, according to Lem.~\ref{lem:the property of 1+sqrt3}, there always exists a feasible flow and the weight of edge from the vertex for $l$ in $V_l$ to $t$ is $k_l$, so $|C\cap S_l|\leq k_l$. 
        
        For $\forall s\in S, d(s,C)\leq (3+\sqrt{3})r^*$, there are two cases: (1) for each point $s\in S$, there exists a point $i\in \Gamma$ such that $d(s,i)\leq 2r^*$ and $i$ is added to $C$, we have $d(s,C)\leq d(s,i)\leq 2r^*$; (2) for each point $s\in S$, there exists a point $i\in \Gamma$ such that $d(s,i)\leq 2r^*$, however, a point $j\in \Gamma_{l}, l\in[m]$ replace $i$ to be added to $C$ as $d(i,j)\leq (1+\sqrt{3})r^*$, then $d(s,C)\leq d(s,j)\leq d(s,i)+d(i,j)\leq 2r^*+(1+\sqrt{3})r^*=(3+\sqrt{3})r^*$. That is, there exists a center set $C\subseteq \bigcup_l{\Gamma_l} \cup \Gamma$ satisfying the fairness requirement and covering all the points of $S$ within a radius $(3+\sqrt{3})r^*$. 

        If there exists a center set $C\subseteq \bigcup_l{\Gamma_l} \cup \Gamma$ satisfying the fairness requirement and covering all the points of $S$ within a radius $(3+\sqrt{3})r^*$. As $|C|=|\Gamma|\leq k$, there exists a flow of value equal to $|\Gamma|$ in the auxiliary graph $G$.
\end{proof}

\begin{theorem}
    For the center set $C$ created by Alg.~\ref{alg:jones_based_stream}, we have: (1) $|C\cap S_l|\leq k_l$; (2) $|C|\leq k$; (3) $d(s,C)\leq (3+\sqrt{3})r^*$ holds for each $s\in S$. In other word, Alg.~\ref{alg:jones_based_stream} achieves the approximation ratio $3+\sqrt{3}\approx 4.732$ for the streaming fair $k$-center problem for general $m$.  
\end{theorem}

\begin{proof}
    For Cond.~(1), as the construction for graph $G$, the weight of the edge from the vertex for $l$ in $V_l$ to $t$ is $k_l$, the center set $C$ created by Alg.~\ref{alg:jones_based_stream} satisfies $|C\cap S_l|\leq k_l$.

    For Cond.~(2), first, $|\Gamma|\leq k$ holds according to lem.~\ref{lem:the size of gamma is less than k}, then, the weight of the edge from the vertex $a\in \Gamma$ to the vertex $f$ in $V_f$ is $1$ and the weight of the edge from the vertex $f$ in $V_f$ to the vertex $l$ in $V_l$ is $1$. According to the conservation law of flow in the maximum flow algorithm, $|C|\leq k$ holds.

    For Cond.~(3), we bound the distance from any point $s\in S$ to $C$ created by Alg.~\ref{alg:jones_based_stream}. For each point $s\in S$, there are two cases: (1) There exists a point $i\in \Gamma$ such that $d(s,i)\leq 2r^*$, and $i$ is added into $C$ as a center, thus $d(s,C)\leq 2r^*$; (2) There exists a point $i\in \Gamma$ such that $d(s,i)\leq 2r^*$, then there exists a point $j\in \Gamma_l, l\in[m]$, $d(i,j)\leq (1+\sqrt{3})r^*$ and $i$ is replaced by $j$ to add to center set $C$, we have $d(s,C)\leq d(s,j)\leq d(s,i)+d(i,j)\leq 2r^*+(1+\sqrt{3})r^*=(3+\sqrt{3})r^*$. 
\end{proof}

Following the above theorem, and accounting for the cost of guessing \(r^*\) over the stream, we eventually obtain a streaming algorithm with an approximation ratio  \((3+\sqrt{3}+\epsilon)\) and a memory complexity $k\log_{\epsilon} \alpha$  by running Alg.~\ref{alg:jones_based_stream} for each guessed value of \(r^*\) with binary search. Moreover, by using the smaller threshold \(\lambda=\sqrt{2}r^*\) to generate \(\Gamma_l\) during streaming and modifying Alg.~\ref{alg:jones_based_stream} accordingly via replacing \(\sqrt{3}r^*\) with \(\sqrt{2}r^*\), the approximation ratio can be improved to $4.42$, which slightly improves upon the previous state-of-the-art ratio of \(4.464\)~\cite{guo2026fair}.

\section{Experimental Results}
\label{sec:experiment}
In this section, we  evaluate our algorithms on both simulated and real-world datasets, comparing them to three approximation algorithms as baselines. All experiments were averaged over multiple runs, implemented in Python 3.8, and executed on a 12th Gen Intel(R) Core(TM) i9 with 64 GB of RAM.\footnote{Our code can be found in GitHub~\url{https://github.com/ChaoqiJia/FPT_EuclideanFairk-Center}.}

\subsection{Experimental Setting}
\paragraph{Datasets.} Following the approach in previous work~\cite{kleindessner2019fair}, we used their method to construct a simulated dataset with a known optimal solution for the $k$-center problem. In addition, we apply our algorithms to three real-world datasets from UCI~\cite{asuncion2007uci}: Wholesale, Student and Adult. Following the previous work~\cite{jones2020fair,chen2019proportionally,guo2026improved}, we utilized numeric features for clustering and selected multiple categorical attributes to construct datasets with the fair requirement.

\paragraph{Algorithms.}

To evaluate our offline parameterized algorithm, which achieves a $(1+\sqrt{3})$-approximation ratio (OPA; see Section~\ref{subsec:off_approx}), we compare it against two baseline methods: the $5$-approximation algorithm~\cite{kleindessner2019fair} and the $3$-approximation algorithm~\cite{jones2020fair}.

To assess the clustering quality of our online algorithms for the fair $k$-center problem, we implement a $(1+2\sqrt{3}+\epsilon)$-approximation parameterized streaming algorithm (PSA; see Section~\ref{subsec:approx-1_2sqrt3}) and an improved algorithm with $1+2\sqrt{2}+\epsilon$ ratio (imPSA), and a $(3+\sqrt{3}+\epsilon)$-approximation algorithm with polynomial time (denoted as $(3+\sqrt{3})$-approx; see Section~\ref{sec:edu_fair}). We compare these approaches with two online baselines: the $(7+\epsilon)$-approximation algorithm~\cite{lin2024streaming} and the $(5+\epsilon)$-approximation algorithm~\cite{guo2026improved}.

\paragraph{Constraints Settings.}
Following the fairness principle of disparate impact as outlined by~\citeauthor{feldman2015certifying}, we restricted the selection to $k_l$ data points from the $l$th group to serve as centers. We then evaluated the clustering quality by varying the parameter $m$ and the fair ratio across these datasets.

\paragraph{Metrics.}
We use the \textit{cost} metric, as defined in Section~\ref{sec:pre}, to compare the quality of clustering across the datasets based on their average values. Additionally, we measure runtime in seconds.

\begin{table*}[t]
\centering
\small
\resizebox{0.88\textwidth}{!}{
\begin{tabular}{l|l|cccccc}
\toprule
\textbf{Algorithms} & \textbf{Approx. Ratio} & \textbf{A-Gender} & \textbf{A-Race} & \textbf{S-Address} & \textbf{S-School} & \textbf{S-Sex}  & \textbf{W-Location} \\
\midrule
\textbf{OPA} & $1+\sqrt{3}$ & -- & -- & \textbf{1.356} & \textbf{1.376} & \textbf{1.349} & \textbf{0.633} \\
\cite{jones2020fair} & $3$  & 0.212 & 0.232 & 1.726 & 1.668 & 1.663 & 0.749 \\
\cite{kleindessner2019fair} & $5$ & 0.226 & 0.226 & 1.702 & 1.578 & 1.834 & 0.749 \\
\midrule
\textbf{imPSA} & $1+2\sqrt{2}+\epsilon$ & - & -  & \textbf{1.502} & \textbf{1.623} & \textbf{1.503} & \textbf{0.595} \\
\textbf{PSA}& $1+2\sqrt{3}+\epsilon$  & 0.243 & 0.243 & {1.568} & {1.779} & 1.633 & {0.632} \\

\textbf{$3+\sqrt{3}$-Approx}  & $3+\sqrt{3}+\epsilon$ & \textbf{0.234} & \textbf{0.229} & 1.684 & 1.832 & {1.540} & 0.812 \\
\cite{guo2026improved}& $5+\epsilon$   & 0.307 & 0.296 & 1.834 & 1.788 & 1.834 & 0.844 \\
\cite{lin2024streaming}& $7+\epsilon$  & 0.310 & 0.701 & 1.823 & 1.788 & 1.834 & 0.911 \\
\bottomrule
\end{tabular}
}
\caption{Evaluation of clustering cost for fair 
$k$-center on real-world datasets.}
\label{tab:realdata}
\end{table*}

\subsection{Experimental Analysis}
\label{subsec:Exper}
In this subsection, we analyze the approximation ratios and running times of all algorithms on a small simulated dataset by varying the number of groups $m$, with focus on our proposed offline and streaming methods: $(3+\sqrt{3})$-Approx, OPA, PSA, and imPSA. We then compare the algorithms on three real-world datasets in terms of clustering cost. Finally, we evaluate how the clustering cost varies under different fairness requirement ratios.

\paragraph{Approximation Factor.}
We compare algorithm performance by evaluating the relative solution ratio against the provided optimal radius $r^*$ on the simulated dataset. The ratio of the evaluation result can be referred to as \textit{empirical approximation ratio}, and the maximum value represents the worst-case. Our target in this experiment is to validate the approximation factors achieved by our algorithms ($(3+\sqrt{3}+\epsilon)$-Approx, OPA, PSA and imPSA). In Figure~\ref{fig:simu_approx}, we compared our algorithms with the baselines in the settings of $|S| = 200$, $k=9$, and increased the number of groups $m$ from $2$ to $6$. 

Across all values of the number of groups $m$, the offline methods (solid-color boxes) consistently achieve lower empirical approximation ratios than the streaming methods (hatched boxes). This holds even though some offline baselines (e.g., Kleindessner et al.) have weaker theoretical guarantees than PSA ($(1+2\sqrt{3}+\epsilon)$-approx) and $(3+\sqrt{3}+\epsilon)$-approx, while only imPSA breaks the observation. We attribute the stronger empirical performance of the offline methods to their access to the full dataset, whereas streaming algorithms must operate under strict memory constraints and can only exploit a limited subset of data points. This interpretation is further supported by imPSA, which uses more memory than the other streaming algorithms and correspondingly achieves better empirical performance. In addition, our results report empirical performance averaged over $10$ runs on $10$ simulated datasets for each group size $m$; the worst-case instances that drive the theoretical approximation guarantees may not arise in our experiments.

When we analyze the offline and streaming results separately, the empirical approximation ratios are consistent with the theoretical guarantees: algorithms with smaller approximation factors tend to yield smaller observed ratios, and the results remain well within the corresponding worst-case bounds. This agreement indicates strong alignment between our empirical findings and theoretical analysis. Moreover, the experimental results exhibit a similar trend as we vary the number of groups. This suggests that the empirical performance observed for $m=2$ generalizes well to larger values of $m$.

The parameterized methods (i.e., OPA, PSA and imPSA) generally achieve better empirical approximation ratios compared with the baselines. In particular, OPA ($1+\sqrt{3}$-Approx) performs best among the offline algorithms, while imPSA ($1+2\sqrt{2}$-Approx) typically outperforms the generic streaming baselines (\citeauthor{guo2026improved,lin2024streaming}). We reason that the parameterized algorithm spent more time and more memory space optimizing the selection of centers.

\paragraph{Clustering on Real World Datasets.}

In Table~\ref{tab:realdata}, we evaluate clustering cost on three real-world datasets under different fairness group settings, setting the number of centers to $k = 1\%$ of the dataset size. Under this setting, the \textit{Adult} yields a relatively large $k$ (about $300$), whereas the \textit{Student} and \textit{Wholesale} datasets have $k$ less than $10$. Consistent with the simulated data, these algorithms exhibit similar empirical trends across datasets. We attribute the strong performance of our methods in part to the larger pool of available candidate centers. However, we also show that the optimal solution restricted to the candidate center set need not coincide with the best performance on the full dataset, especially when $k$ is large. For instance, $(3+\sqrt{3})$-Approx occasionally achieves slightly lower cost than PSA on \textit{Adult}.

\begin{figure}[t]
  \centering
  \includegraphics[width=0.86\linewidth]{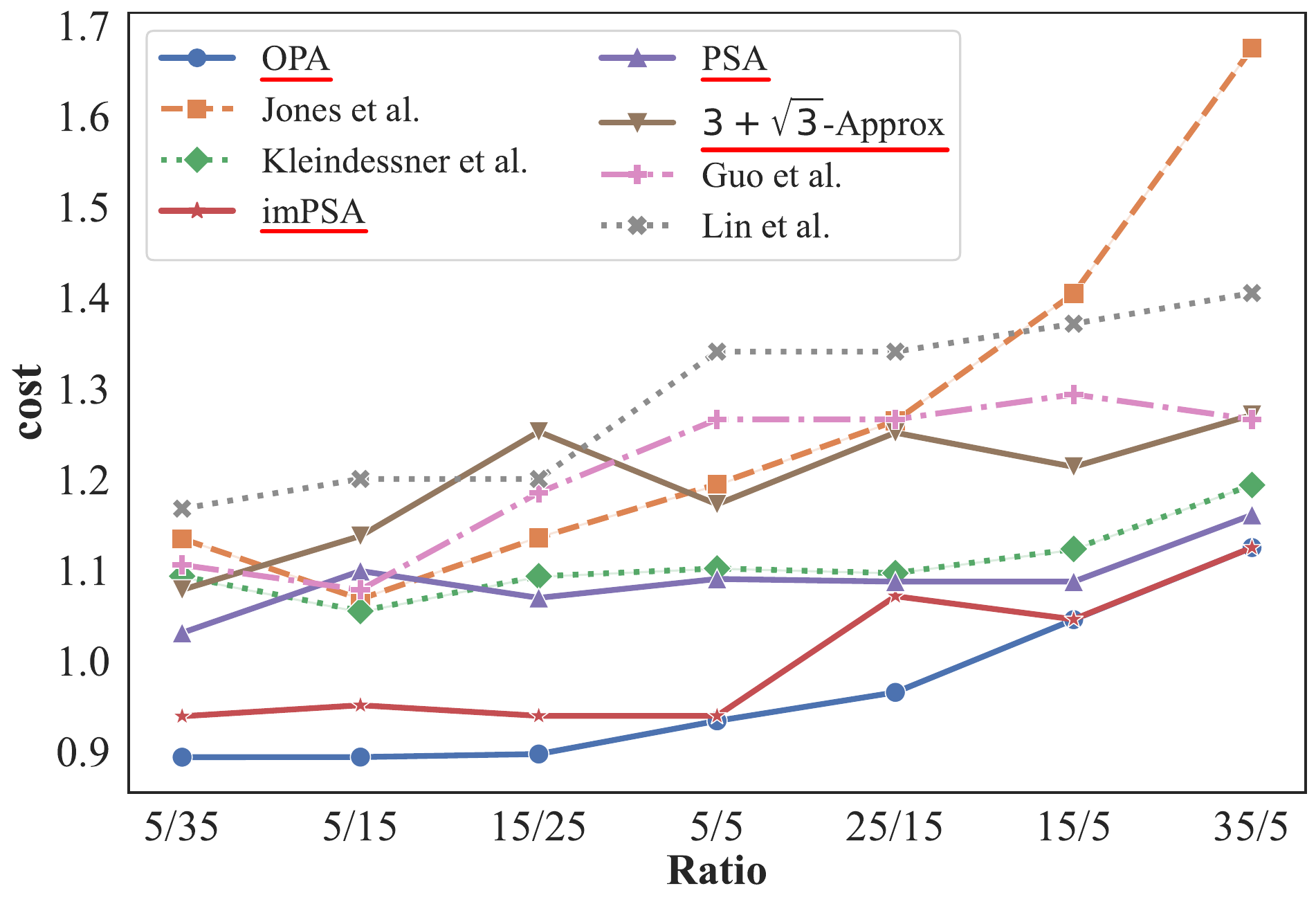} 
  \caption{Cost of our algorithms (indicated by a red underline) compared with other baselines with the different balanced fair center size.}
  \label{fig:robust}
\end{figure}

\paragraph{Robustness.}
In Figure~\ref{fig:robust}, we vary the fairness constraint ratio on the \textit{Student} dataset using the \textit{school} attribute. We start from the ratio induced by the data distribution when selecting $10\%$ of the points as centers, which corresponds to a constraint range of $5/35$ in this dataset, and then gradually adjust the ratio until it reaches $35/5$.

In Figure~\ref{fig:robust}, we observe that all algorithms are affected by the fairness ratio: the clustering cost increases as the ratio becomes more unbalanced. We attribute this trend to the fact that our methods first construct a candidate center set without using group labels. Thus, when the required center ratio aligns with the dataset's natural group distribution, the candidate set is more likely to cover points from each group in a way that satisfies the fairness constraints, which is also representative of many real-world settings. In contrast, this experiment considers deliberately mismatched ratios that deviate from the data distribution. In this case, the cost increases, but the relative ranking of the algorithms remains similar.

\section{Conclusion}

 In this paper, we first propose a parameterized approximation algorithm with a ratio of $1+\sqrt{3}\approx2.732$ for the offline Euclidean fair $k$-center clustering problem. By applying this algorithm in the post-streaming processing stage, we obtain a streaming algorithm with a ratio of $1+2\sqrt{3}+\epsilon\approx4.464$. These ratios can be further improved to $1+\sqrt{2}\approx2.414$ and $1+2\sqrt{2}+\epsilon\approx3.828$, respectively. To ensure polynomial running time, we devise another one-pass streaming algorithm that incorporates network flow techniques into the milestone algorithm of~\cite{jones2020fair}, achieving an approximation ratio of $3+\sqrt{3}+\epsilon\approx4.732$. This ratio can be further improved to $4.42$, improving upon the previous state-of-the-art ratio $4.464$ in Euclidean space due to~\citeauthor{guo2026fair}.

\clearpage

\section*{Acknowledgments}

This work is supported by National Natural Science Foundation of China (No. 12271098) and Key Project of the Natural Science Foundation of Fujian Province (No. 2025J02011). 

\bibliographystyle{named}
\bibliography{ijcai26}

\clearpage
\appendix

\section{Proof of Lem.~\ref{lem:uppertbinom}}
 By Stirling's approximation 
$\log(n!)=n\log n-n\log e+O(\log n)$, we have
\[
\log\ensuremath{\tbinom{2mk}{k}}=\log\frac{(2mk)!}{k!((2m-1)k)!}.
\]
That is,
\begin{align*}
\log\ensuremath{\tbinom{2mk}{k}} & = &  & \log((2mk)!)-\log(k!((2m-1)k)!)\\
 & = &  & k\cdot(2m-1)(\log(2m)-\log(2m-1))\\
 &  &  & +k\log2m+ O(\log(2mk))
\end{align*}

For bounding the term $(2m-1)(\log(2m)-\log(2m-1))$ of the above
inequality, we need only to show that the following always holds for any
$t\geq1$:

\begin{equation}
f(t)  =  \frac{1}{t}-(\ln(t+1)-\ln(t))\geq 0.\label{eq:ft-bound}
\end{equation}

By calculation, we have
\begin{eqnarray*}
g(t) & = & e^{\frac{1}{t}}-e^{\ln(t+1)-\ln(t)}\\
 & \geq & \left(1+\frac{1}{t}+\frac{1}{2t^{2}}\right)-\left(1+\frac{1}{t}\right)\\
 & \geq & 0,
\end{eqnarray*}
where the first inequality is from Talor expansion. Then from $g(t)\ge 0$ as above,  we have $f^{\prime}(t)\geq0$ holds for any $t\geq1$, and consequently conclude that $f(t)\geq0$ is true for any $t\geq1$.

Hence,
\begin{eqnarray*}
& &(2m-1)(\log(2m)-\log(2m-1))  \\
&= & \frac{1}{\ln2}(2m-1)(\ln(2m)-\ln(2m-1))\\
&\le  & \frac{1}{\ln2},
\end{eqnarray*}
where the first inequality is from Eq.~(\ref{eq:ft-bound}).
Therefore, we have 
\[\log\ensuremath{\tbinom{2mk}{k}}\leq  \frac{1}{\ln2}+k(1+\log m) + O(\log(2mk)),\]
and hence
\[
\tbinom{2mk}{k}\leq2^{{\ln2}+k (1+\log m) +O(\log(2mk))}=O(2^{k(1+\log m)}\cdot mk).
\]
\end{proof}

\section{ILP}

Furthermore, we note that to exactly solve the fair $k$-center problem in Stage 2), it is more efficient and elegant to use the following ILP than enumeration:
\begin{align*}
\min &  &  & \eta\\
s.t. &  &  & x_{ij}\leq y_{i} &  & \forall i,j\in S\\
 &  &  & \sum_{i}y_{i}\leq k\\
 &  &  & \sum_{i\in S_{l}}y_{i}\le k_{l} &  & \forall l\in [m]\\
 &  &  & \sum_{i\in S}x_{ij}\geq1 &  & \forall j\in S\\
 &  &  & d_{ij}x_{ij}\leq\eta &  & \forall i,j\in S\\
 &  &  & x_{ij,}y_{i}\in\{0,1\}
\end{align*}
where $d_{ij}$ be the distance between $i$ and $j$, $y_{i}$ indicates whether point $i$ is selected as a center, and  $x_{ij}$ indicates whether point $j$ is assigned to center $i$.

\section{\texorpdfstring{Dealing with Unknown $r^*$}{Dealing with Unknown r*}}

Throughout this paper, we assume that the optimal radius $r^{*}$ is known. However, since the exact value of $r^*$ is unknown,  we introduce a method to find an equivalent value that serves as a suitable replacement. This approach follows the same line as the previous elegant algorithms proposed in several previous works, including \cite{charikar1997incremental,badanidiyuru2014streaming,matthew2008streaming}.

We start with the observation below:

\begin{lemma}
    Let $\Psi=\{d_{ij}\mid i,j\in S\} $ be the set of distances between any two points in $S$. Then, we have $r^*\in \Psi$. That means that the value of the optimal radius must be a distance between two points of $S$.
\end{lemma}

Since searching through all the distances between points in a streaming model is not feasible, we use a modified version of the doubling algorithm by Charikar et al. \cite{charikar1997incremental}, which is used for incremental clustering.

For a given parameter $\varepsilon$, the key idea is to maintain a lower bound $L$ during the streaming algorithm and use $r_j=L\left(1+\varepsilon\right)^{j}$, $j\in \{1,\dots,t\}$ as a replacement for $r^*$. Initially, $L$ is the minimum distance among the first $k+1$ arrived points and may increase in later iterations.   The value of $t$ is given by  $t=\lceil -\frac{\log\varepsilon}{\log(1+\varepsilon)}\rceil$\footnote{By setting the value of $t$, we ensure that $r_1 \leq \varepsilon \cdot r_t $.}. Note that there are $O(t)$ possible replacements, as $j$ takes $t$ values. Thus, we have $t$ instances running in parallel, consuming $O(tk)=O(k\frac{1}{\varepsilon}\ln \frac{1}{\varepsilon})$ memory. For convenience, we use $C_j$ to denote the set of centers selected according to  $r_j$, and $\Phi=\bigcup_{j=1}^t C_j$. 

It remains to update $L$. The key idea is to ensure that the value of $r_1$ is sufficiently large, i.e. there are no more than $k$ centers in $C_1$ according to $r_1$. When there are more than  $k$ centers in $C_j$, we increase $t$ to $t'$ and $L/2$ to $(L/2)\left(1+\varepsilon\right)^{t'-t}$, ensuring that the new radius satisfies $L/2\leq r\leq (1+\varepsilon)L/2$. The center set $C_j$ for new $r_j$ is then produced by employing the famous greedy algorithm \cite{hochbaum1985best}, applied to the cached points of $\Phi$ plus the currently arriving point.

For the correctness, we have the following lemma:

\begin{lemma}
    For any given $\varepsilon>0$, the center set $C_j$ with size bounded by $k$ produced using minimum $r_j$ among the $t$ values is a $\lambda$-independent center set for $\lambda= (2+\varepsilon)r^*$.  
\end{lemma}

\section{\texorpdfstring{Better Ratios under Smaller $\lambda$}{Better Ratios under Smaller lambda}}
\label{subsec:imPSA}
According to 
Lem.~\ref{lem:sqrt3d-independent center set},
we get that when $\lambda$ decreases, the size of $\Gamma$ grows. Moreover,  smaller $\lambda$ indicates a better ratio and a larger size of $\Gamma$ indicates larger memory complexity.

\begin{lemma}
    \label{lem:sqrt3d-independent center set1} 
    Assume  $\Gamma\subseteq S$ is a $\lambda$-independent center set in Euclidean space.
For any positive interger $h$, if $\lambda = \sqrt{2+\frac{2}{h}}\cdot r^{*}$, then $|\Gamma|\leq h\cdot k$. 
\end{lemma}

\begin{table*}[!h]
\footnotesize
\begin{tabularx}{\linewidth}{XXXll}
\toprule[1pt]
\textbf{Dataset} & \textbf{\#Records} & \textbf{\#Dimension} & \textbf{Feature} & \textbf{\#Groups}\\
\midrule[0.8pt]
Wholesale  & 440 & 6 & region = (77, 47, 316) & 3\\
\midrule
Student & 649 & 16 & school = (423, 226)& 2\\
&  &  &  address = (197, 452)& 2\\
&  &  &  sex = ( 383, 266)& 2\\
\midrule
Adult & 32,561 & 6 & gender = (10771, 21790)& 2\\
&  &  & race = (311, 1039, 3124, 271, 27816) & 5\\
\midrule[1pt]
\end{tabularx}
\caption{Datasets Summary.\label{tab:datasets}}
\end{table*}

We can prove the above lemma by generalizing the proof of  Lem.~\ref{lem:sqrt3d-independent center set}, and employing the following famous property in computational geometric \cite{grunbaum1967convex}: 

\begin{lemma}
    \label{lem:hpoints} 
 For a $\delta$-dimensional ball (sphere) with radius $r^*$ in an Euclidean space, for any positive integer $h$, $\delta\ge h\geq 2$, there exist at most $h$ points in the ball such that every pair of points is with a distance larger than $\sqrt{2+\frac{2}{h}}\cdot r^{*}$.
\end{lemma}

In Fig.~\ref{fig:thesmallestgamma},  there exist at most $4$ points with piece-wise distance equals $\frac{4}{\sqrt{6}}r^*$ in a  $3$-dimensional ball with radius $r^*$. Then, there can be at most $h=4-1$ points buffered after streaming for each optimum cluster as a $\delta$-dimensional ball. Therefore, there are $O(3k)$ points in the buffer when setting $\lambda=\frac{4}{\sqrt{6}}r^*$.

\begin{figure}
    \centering
    \includegraphics[width=0.26\textwidth]{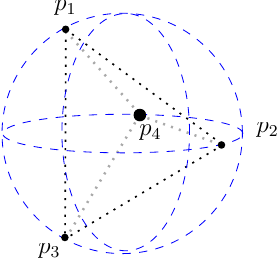}
    \caption{An example with smaller $\lambda$ ($\lambda=\frac{4}{\sqrt{6}}r^*$):  In a 3-dimensional ball with radius $r^*$, there are at most $4$ points with piece-wise distance equals $\frac{4}{\sqrt{6}}r^*$. In other words, there can be at most $h-1=3$ points with piece-wise distance strictly larger than $\frac{4}{\sqrt{6}}r^*$.}
    \label{fig:thesmallestgamma}
\end{figure}

Similar to Lem.~\ref{lem:1plus2sqrt3}, we have the following property by setting $\epsilon=\sqrt{2+\frac{2}{h}}-\sqrt{2}$:

\begin{lemma}\label{lem:bestratioforE}
The fair $k$-center problem admits a parameterized algorithm with an approximation ratio $1+\sqrt{2}+\epsilon\approx 2.414$ and a runtime $O(nk^2\left(\frac{m}{\epsilon}\right)2^{k(1+\log \frac{m}{\epsilon})} \log n)$ for any small $\epsilon>0$. Moreover, it admits a parameterized streaming algorithm with a ratio $1+2\sqrt{2}+\epsilon\approx 3.828$ and  $\frac{k\log \alpha}{\epsilon}$ memory.
\end{lemma}

\section{Extended Experiential Results}
\label{sec:exper}
\subsection*{Datasets}

 Following the previous work~\cite{jones2020fair,chen2019proportionally}, we evaluate algorithms on three datasets summarized in Table~\ref{tab:datasets} and the details are described below.

\descr{Wholesale}\footnote{\url{https://archive.ics.uci.edu/dataset/292/wholesale+customers}} dataset contains monetary spending on 440 (\#Records) products of clients of a wholesale distributor. We select the \textit{channel} (Horeca: 298, Retail: 142) attribute for group assignments.
 
\descr{Student}\footnote{\url{https://archive.ics.uci.edu/dataset/320/student+performance}} dataset includes 649 (\#Records) information about grades, socioeconomic, and school data relevant to predicting the academic performance of students in Math. We use features \textit{sex} (Male: 226, Female: 383) for group assignments.

\descr{Adult}\footnote{\url{https://archive.ics.uci.edu/dataset/2/adult}} dataset contains socioeconomic 32,561 (\#Records) individuals for prediction of whether income exceeds 50k/year. We use \textit{gender} (Female: 10,771, Male: 21,790) for group assignments.

\subsection{Runtime(s)}
\label{subsec:runtime}

\begin{figure}[t]
  \centering
  \includegraphics[width=\linewidth]{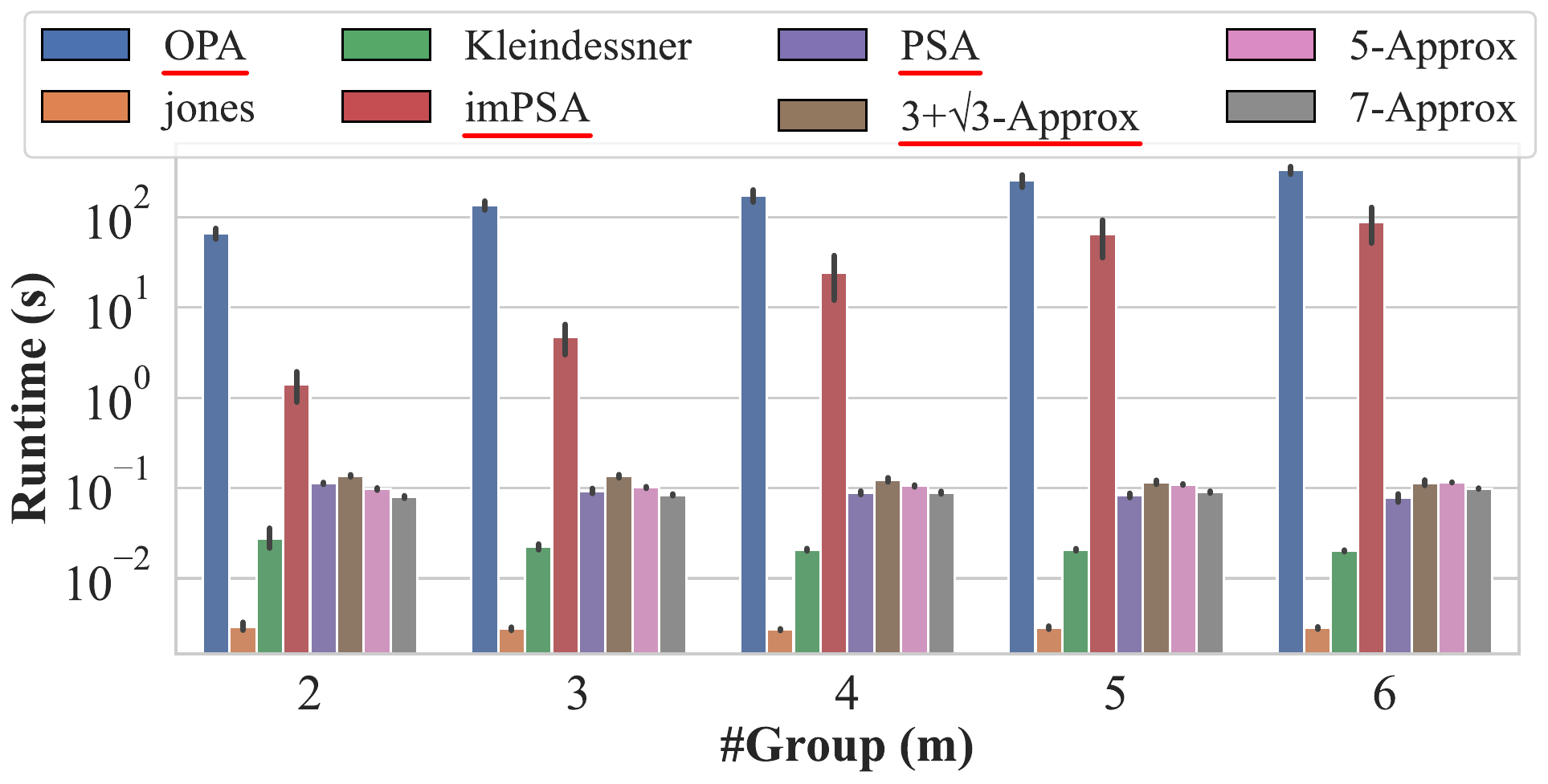} 
  \caption{Running time of our algorithms (indicated by a red underline) compared with other baselines.}
  \label{fig:simu_runtime}
\end{figure}

On the simulated dataset, we report the running times of our algorithms and the baselines in Figure~\ref{fig:simu_runtime}. We observe that OPA and imPSA are slower than the other methods, since they are parameterized algorithms, consistent with our theoretical analysis. In contrast, PSA exhibits a running time comparable to the other streaming algorithms. We attribute this to the fact that all streaming methods are sensitive to the center size $k$, which is small in our synthetic setting.
Moreover, all of our online algorithms are slower than those of \cite{jones2020fair} and \cite{kleindessner2019fair} on this small dataset, because the streaming implementations must read and process the input by line, incurring additional I/O and initialization overhead.

\end{document}